\documentclass{article}

\usepackage{microtype}
\usepackage{graphicx}
\usepackage{subfigure}
\usepackage{booktabs} 
\usepackage{amssymb,graphicx,mathrsfs,epsfig,epstopdf}
\usepackage[algo2e,algonl,ruled]{algorithm2e}
\usepackage{tikz}
\usepackage{pgf}
\usepackage{appendix}
\usetikzlibrary{automata,fit,decorations.pathmorphing}
\usetikzlibrary{positioning,shapes.multipart}
\usepgflibrary{shapes.multipart}
\usepackage{macros}
\usepackage[latin1]{inputenc}
\usepackage{amsmath}
\usepackage{amsfonts}
\usepackage{float}
\usepackage{hyperref}

\usepackage[accepted]{icml2019}

\icmltitlerunning{PLUME: Polyhedral Learning Using Mixture of Experts}

\begin{document}
\def \bw {\tilde{\mathbf{w}}}
\def \by {\tilde{\mathbf{y}}}
\def \bx {\tilde{\mathbf{x}}}
\def \bz {\tilde{\mathbf{Z}}}
\def \bX {\tilde{\mathbf{X}}}
\def \xx {\mathbf{x}}
\def \I {\mathbb{I}}
\def \R {\mathbb{R}}
\def \K {\mathbf{k}}
\def \yy {\mathbf{y}}
\def \zz {\mathbf{z}}
\def \ee {\mathbf{e}}
\def \gg {\mathbf{g}}
\def \ww {\mathbf{w}}
\def \dd {\mathbf{d}}
\def \Rh {\hat{R}}
\newcommand{\N}{{\mathbb N}}
\newcommand{\E}{{\mathbb E}}
\def \amax {\operatorname{argmax}}
\def \amin {\operatorname{argmin}}
\def \muu  {\mbox{\boldmath $\mu$}}
\def \xii  {\mbox{\boldmath $\xi$}}
\def \alphaa  {\mbox{\boldmath $\alpha$}}

\newtheorem{theorem}{Theorem}
\newtheorem{definition}[theorem]{Definition}
\newtheorem{result}[theorem]{Result}
\newtheorem{remark}[theorem]{Remark}
\newtheorem{lemma}[theorem]{Lemma}
\newtheorem{corollary}[theorem]{Corollary}
\newtheorem{Remark}[theorem]{Remark}
\newtheorem{assumption}[theorem]{Assumption}
\newcounter{MYtempeqncnt}
\newenvironment{proof}{Proof:}{\hfill \tikz \draw[fill] (0,0) rectangle (0.25,0.25);}

\twocolumn[
\icmltitle{PLUME: Polyhedral Learning Using Mixture of Experts}




\begin{icmlauthorlist}
\icmlauthor{Kulin Shah}{IIIT}
\icmlauthor{P. S. Sastry}{IISc}
\icmlauthor{Naresh Manwani}{IIIT}
\end{icmlauthorlist}

\icmlaffiliation{IIIT}{Machine Learning Lab, KCIS, IIIT Hyderabad, India}
\icmlaffiliation{IISc}{Electrical Engineering, IISc Bangalore, India}

\icmlcorrespondingauthor{Kulin Shah}{kulin.shah@students.iiit.ac}
\icmlcorrespondingauthor{P. S. Sastry}{sastry@ee.iisc.ernet.in}
\icmlcorrespondingauthor{Naresh Manwani}{naresh.manwani@iiit.ac.in}


\vskip 0.3in
]



\printAffiliationsAndNotice{} 

\begin{abstract}
In this paper, we propose a novel mixture of expert architecture for learning polyhedral classifiers. We learn the parameters of the classifier using an expectation maximization algorithm. We derive the generalization bounds of the proposed approach. Through extensive simulation study, we show that the proposed method performs comparably to other state-of-the-art approaches.
\end{abstract}

\section{Introduction}
In a binary classification problem, if all the class $C_+$ examples are concentrated in a single convex region with the class $C_-$ examples being all around that region,
then the region of class $C_+$ can be well captured by a polyhedral set. A polyhedral set is a convex set which is formed by an intersection of a finite number of closed halfspaces \cite{Rockafellar}. An essential property of polyhedral sets is that they can be used to approximate any convex connected subset of $\R^d$.
This property of polyhedral sets makes the learning of polyhedral regions an interesting problem in pattern recognition. Polyhedral classifiers are useful in many real-world applications, e.g., text classification \citep{SatiO18}, cancer detection \citep{Murat2008}, visual object detection, and classification \citep{Cevikalp2017PolyhedralCC} etc.

To learn a classifier in this case, we need to find a closed connected set (e.g., an enclosing ball) which contain all positive examples
leaving all the negative examples outside the set. Support vector data description method \cite{Tax2004} does this task by fitting a minimum enclosing hypersphere
in the feature space to include most of the class $C_+$ examples inside the hypersphere while considering all the class $C_-$ examples as outliers.
In such techniques, the nonlinearity in the data is captured by choosing an appropriate kernel function.
With a non-linear kernel function, the final classifier may not provide good geometric insight on the
class boundaries in the original feature space. Such insights are useful to understand the local behavior of the classifier in different regions of the
feature space.

A well-known approach to learn polyhedral sets is the top-down decision tree method. In a binary classification problem, a top-down decision tree represents
each class region as a union of polyhedral sets \cite{Breiman1984, Duda, Naresh11}. Here all positive examples belong to a single polyhedral set. However, top-down decision tree algorithms, due to greedy nature, may not learn a single polyhedral set well.

Neither SVM nor top-down decision trees (CART) can learn classifier as a polyhedral set. Even when we want a classifier for general data, neither CART nor SVM can learn a classifier representable as compactly as that by PLUME. In the context of explainable AI, a classifier whose decision is based on two hyperplanes in the original feature space is certainly more understandable than a large decision tree or a classifier that is a linear combination of many kernel functions. 

Unlike such general purpose approaches, there are many specialized approaches for learning polyhedral classifiers. In such methods, we first fix the structure of the polyhedral classifier and determine the optimal parameters of this fixed structure. In the case of polyhedral classifiers, we can adjust the structure by
choosing the number of hyperplanes. 

One can formulate a constrained optimization problem to learn polyhedral classifiers \cite{Astorino,Murat2008,Carlotta, SatiO18}.
This optimization problem minimizes the sum of classification errors over the training set subject to the separability conditions. Conceptually, learning a polyhedral set requires learning each of the hyperplanes constituting it. But we cannot solve these linear problems (of learning individual hyperplanes) separately because the available training set cannot be easily transformed into training sets for learning individual hyperplanes. While points of $C_+$ would be positive examples for each of the linear problems, it is not known apriori which points of $C_-$ are negative examples for learning which hyperplane. In \cite{Astorino}, this problem is solved by first enumerating all possibilities for misclassified negative examples (e.g., which hyperplane
is responsible for a negative example to get misclassified and for each negative example there could be many such hyperplanes) and then solving a linear program
for each possibility to find descent direction. This approach becomes computationally very expensive.
 \cite{Murat2008} assume that for every linear subproblem, a small subset of negative examples is known and propose a cyclic
optimization algorithm. Their assumption of knowing a subset of negative examples corresponding to
every linear subproblem is not realistic in many practical applications. \citet{Zhou:2016} propose a method in which the positive class is enclosed using the intersection of non-linear surfaces using kernel methods. By using linear kernels it becomes same as polyhedral learning. However, they use the same objective function as other constrained optimization problems. 

\citet{Naresh10} propose a logistic function based posterior probability model for polyhedral learning. They learn the parameters using alternating minimization. The method is shown to perform well experimentally though there are no theoretical guarantees about convergence or generalization errors. A large margin framework for polyhedral learning is discussed in \cite{NIPS2014_5511} using stochastic gradient descent. 

In this paper, we propose a mixture of experts model for learning polyhedral classifiers. The mixture of experts \citep{Nowlan1} model contains several linear experts (classifiers) and each of the expert champions one of the regions of the feature space. It also includes a gating function which decides which expert to use for a particular example. Even though mixture of experts is a generic approach, in the context of learning polyhedral classifiers it carries a unique structure which requires a lesser number of parameters. We see that it does not need separate parameters to model the gating function. It uses experts parameters themselves for modeling the gating function. 
As far as our knowledge is concerned, this is the first attempt in this direction. We make the following contributions in this paper.
\begin{enumerate}
    \item We propose a novel mixture of experts architecture to model polyhedral classifiers. We propose an expectation maximization (EM) algorithm using this model to learn the parameters of the polyhedral classifiers.
    \item We derive data dependent generalization error bounds for the proposed model with specific constraints that the gating function uses the same parameters as experts.
    \item We do extensive simulations on various datasets and compare with state of the art approaches to show that our approach learns polyhedral classifiers efficiently.
\end{enumerate}
  
The rest of the paper is organized as follows. In Section \ref{Sec:polyhedral classifier},
we state the definitions of polyhedral separability and polyhedral classifiers. We describe the mixture of experts model in Section~\ref{Sec:model} and corresponding EM algorithm in Section~\ref{sec:EM}. We derive the generalization error bounds for the proposed model in Section~\ref{sec:GEB}.
We describe the experimental results in Section~\ref{sec:Exp}. We conclude the paper with some remarks in Section~\ref{conclusions}

\section{Polyhedral Classification}\label{Sec:polyhedral classifier}
Let $S=\{(\xx_1 , y_1),\ldots,(\xx_N,y_N)\}$ be the training dataset,
where $(\xx_n,y_n)\in \R^d\times\{+1,-1\},\forall n$.
Let $C_+$ be the set of
points for which $y_n=1$ and let $C_-$ be the set of points for which $y_n=-1$.

\subsection{Polyhedral Separability}
Two sets $C_+$ and $C_-$ in $\R^d$ are said to be $K$-polyhedral separable if there exists a set
of $K$ hyperplanes having parameters, $(\ww_k,b_k),\;k=1\ldots K$ ($\ww_k \in  \R^d,\;b_k \in \R$), such that

\begin{enumerate}
 \item $\forall \; \xx \in C_+: \;\;\ww_k^T\xx+b_k \geq 0,\;\forall\; k\in\{1,\ldots, K\}$
\item $\forall \; \xx \in C_-: \;\;\exists k\in \{1,\ldots, K\},\;\text{s.t.}\;\ww_k^T\xx+b_k < 0$
\end{enumerate}
This means that two sets $C_+$ and $C_-$ are $K$-polyhedral separable if $C_+$ is contained in a convex polyhedral
set which is formed by intersection of $K$ half spaces and the points of set $C_-$ are outside this polyhedral set. Here all the positive examples satisfy each of a given set of linear inequalities (that defines the half spaces
whose intersection is the polyhedral set). However, each of the negative examples fail to satisfy one (or more) of these
inequalities and we do not know apriori which inequality each negative example fails to satisfy. Thus constraint on each of the negative examples is
logical `OR' of the linear constraints which makes the optimization problem non-convex.

\subsection{Polyhedral Classifier}
Let $(\ww_k,b_k),\;k=1\ldots K$ be the parameters of the $K$ hyperplanes which form the polyhedral set. Here $\ww \in \R^d,b_k \in \R,\;k=1\ldots K$. Let $h:\R^d\rightarrow \R$ be defined as:
\begin{equation}\label{eq:modelfunction}
  h(\xx)=\min_{k\in \{1,\ldots, K\}}(\ww_k^T\xx+b_k)
\end{equation}
Clearly if $h(\xx)\geq 0$, then the condition $\ww_k^T\xx+b_k \geq 0$ is satisfied for all $k\in\{1,\ldots, K\}$ and the point $\xx$ can
be assigned to set $C_+$. Similarly, if $h(\xx)< 0$, there exists at least one $k\in\{1,\ldots,K\}$, for which $\ww_k^T\xx+b_k <0$ and the point $\xx$
can be assigned to set $C_-$. Thus, the polyhedral classifier can be expressed as
\begin{equation}\label{eq:poly-classifier}
 f(\xx)=\text{sign}(h(\xx))=\text{sign}\big{[}\min_{k\in \{1,\ldots, K\}}(\ww_k^T\xx+b_k)\big{]}
\end{equation}
Let $\bw_k =[\ww_k^T~b_k]^T$ and $\bx_n=[\xx_n^T~1]^T$. Now on, we
will express $\ww_k^T\xx+b_k$ as $\bw_k^T\bx$.

\section{Mixture of Experts Model for Polyhedral Classifier}\label{Sec:model}
We propose a new mixture of experts architecture for learning polyhedral classifiers. We model the posterior probability as a mixture of $K$ logistic functions where $K$ is the number of hyperplanes associated with the polyhedral classifier. We write the posterior probability of the class labels as 
\begin{align}
 p_{\Theta}(y|\xx)=\sum_{k=1}^K p(k|\xx)p(y|\xx,k)=\sum_{k=1}^Kg_k(\xx,\Theta)\sigma(y\bw_k^T\bx)\label{eq:posterior}
\end{align}
where $\Theta =[\bw_1^T\;\bw_2^T\;\ldots\;\bw_K^T]^T$ and $\sigma(a)=(1+e^{-a})^{-1}$. Each expert models the posterior probability using the logistic regression. The parameter vector associated with the $k^{th}$ expert is $\bw_k$.
$g_k(\xx,\Theta)$ is the gating function which decides how much weightage should be given to $k^{th}$ expert for an example $\xx$. To ensure that eq.(\ref{eq:posterior}) describes a valid posterior probability model, we choose $g_k(\xx,\Theta)$ such that $\sum_{k=1}^K g_k(\xx,\Theta) =1$ and $g_k(\xx,\Theta)\geq 0,\;\forall k=1\ldots K$.
For learning polyhedral classifiers, we construct the gating function using softmax function as follows:
\begin{equation}
g_k(\xx,\Theta) = \frac{e^{-\gamma \bw_k^T\bx}}{\sum_{j=1}^K e^{-\gamma \bw_j^T\bx}}
\label{eq:soft-gating-model}
\end{equation}
where $\gamma >0$ is a user defined parameter which decides how fast $g_k(\xx,\Theta)$ goes to 0 or 1. Note that the proposed gating function depends on experts parameters only. Moreover,
$$\lim_{\gamma \rightarrow \infty}g_k(\xx,\Theta) = \mathbb{I}_{\left\{k=\underset{j\in\{1,\ldots,K\}}{\arg\min}\bw_j^T\bx\right\}}$$
and hence
$$\lim_{\gamma \rightarrow \infty}p_{\Theta}(y|\xx) = \frac{1}{1+e^{-y\;\underset{j\in\{1,\ldots,K\}}{\min}\;\bw_j^T\bx}}.$$
For a polyhedrally separable data, we know that $y\min_{j\in\{1,\ldots,K\}}\bw_j^T\bx \geq 0$. Thus, $p_\Theta(y=1|\xx)$ will be close to 1 if $\xx \in C_+$. Similarly, $p_\Theta(y=0|\xx)$ will be close to 1 if $\xx \in C_-$. Thus, $p_{\Theta}(y|\xx)$ described in eq.(\ref{eq:posterior}) is a valid probability model for polyhedral learning. Note that the model proposed in \citep{Naresh10} is a limiting case of the model proposed in eq.(\ref{eq:posterior}). The advantage with the proposed model are twofold. As the proposed posterior probability function is a smooth function, the resulting EM formulation will satisfy smoothness conditions required for convergence. This is in contrast to the hard partitioning model in \citet{Naresh10}. The second advantage that we will see is that the proposed model is also better suited for capturing smooth convex boundaries.

\section{EM Algorithm for Learning Polyhedral Classifier}
\label{sec:EM}
 As the posterior probability model is a mixture model, we propose an EM approach for learning the parameters. 
Recall that $S=\{(\xx_1,y_1),\ldots,(\xx_N,y_N)\}$ is the training data set, where
$(\xx_n,y_n)\in \R^d\times\{-1,+1\},\;\forall n$. In the EM framework, we think of $S$ as incomplete data and have $p(y|\xx,\Theta)$ (given by equation (\ref{eq:posterior})) as the model for incomplete data. In this problem, we don't know which expert should be used to classify example $\xx_n$. This is the missing information. We represent the missing information corresponding to example $\xx_n$ as $\zz_n=[z_{n1}~\ldots~z_{nK}]^T$, where each
$z_{nk}\in \{0,1\},k=1\ldots K$, such that $\sum_{k=1}^K z_{nk}=1,\;\forall n$.
Moreover,
\begin{equation}\label{eq:model4}
 P(z_{nk}=1|\xx_n,\Theta)=g_k(\xx_n,\Theta)=\frac{e^{-\gamma\bw_k^T\bx_n}}{\sum_{j=1}^Ke^{-\gamma\bw_j^T\bx_n}}.
\end{equation}
Let $\bar{S}=\{(\xx_1,y_1,\zz_1),\ldots,(\xx_N,y_N,\zz_N)\}$ be the complete data. The complete-data log likelihood is given as follows.
\begin{align*}
 l_{\text{complete}}&(\Theta;\bar{S})=\ln \big{[} \prod_{n=1}^N \prod_{k=1}^K
[P(y_n,z_{nk}|\xx_n,\Theta)]^{z_{nk}}\big{]}\\
&=\sum_{n=1}^N \sum_{k=1}^K z_{nk} \big{[}\ln g_k(\xx_n,\Theta)+\ln \sigma(y_n\bw_k^T\bx_n)\big{]}
\end{align*}

\textbf{E-Step: } In the E-step, we find $Q(\Theta,\Theta^c)$ which is the expectation of complete-data log likelihood.
\begin{align*}
\mathcal{Q}_N(\Theta,\Theta^c)=\E_{\{\zz_1,\ldots,\zz_N\}} \big{[}l_{\text{complete}}(\Theta;\bar{S})|\Theta^c\big{]}\\
= \sum_{n=1}^N \sum_{k=1}^K \big{[}\ln g_k(\xx_n,\Theta)+\ln \sigma(y_n\bw_k^T\bx_n)\big{]}
\pi_k(\xx_n,\Theta^c)
\end{align*}
where $\pi_k(\xx_n,\Theta)=P(z_{nk}=1|\xx_n,y_n,\Theta)$ found as follows.
\begin{align}\label{eq:postprob}
 \nonumber \pi_k(\xx_n,\Theta) &= \frac{g_k(\xx_n,\Theta)\sigma(y_n\bx_n^T\bw_k)}{\sum_{j=1}^K
g_j(\xx_n,\Theta)\sigma(y_n\bx_n^T\bw_j)}\\
&= \frac{e^{-\gamma\bx_n^T\bw_k}\sigma(y_n\bx_n^T\bw_k)}{\sum_{j=1}^K
e^{-\gamma\bx_n^T\bw_j^c}\sigma(y_n\bx_n^T\bw_j)}
\end{align}
It is easy to see that $\mathcal{Q}_N(\Theta,\Theta^c)$ is a concave function of $\Theta$.

\textbf{M-Step: }In the M-step, we maximize $\mathcal{Q}_{N}(\Theta,\Theta^c)$ with respect to $\Theta$ to find the new parameter set
$\Theta^{c+1}$. Since $\mathcal{Q}_N(\Theta,\Theta^c)$ is a concave function of $\Theta$, there exists a unique maxima of it. However, we do not get the closed form solution for the maximization with respect to $\Theta$. Thus, we find $\Theta^{c+1}$ by moving in ascent direction of $\mathcal{Q}_{N}(\Theta,\Theta^c)$ starting
from $\Theta^c$. We can use one of the following approaches.
\begin{itemize}
 \item \textbf{Gradient Ascent: }The gradient ascent update equation is as follows.
\begin{eqnarray}
\label{eq:gradient_ascent}
\Theta^{c+1}
= \Theta^c+\alpha^c \nabla \mathbf{g}^c
\end{eqnarray}
where $\alpha^c$ is step size at iteration $c$ and $\mathbf{g}^c$ is $dK$-dimensional gradient vector at $c^{th}$ iteration. 
$$\mathbf{g}= \begin{pmatrix}
\frac{\partial \mathcal{Q}_{N}(\Theta,\Theta^c)}{\partial \bw_1}^T\;\;
\frac{\partial \mathcal{Q}_{N}(\Theta,\Theta^c)}{\partial \bw_2}^T \ldots
\frac{\partial \mathcal{Q}_{N}(\Theta,\Theta^c)}{\partial \bw_K}^T
\end{pmatrix}^T$$
$\frac{\mathcal{Q}_{N}(\Theta,\Theta^c)}{\partial \bw_k}$ can be found as follows.
\begin{align}
\nonumber &\frac{\partial \mathcal{Q}_{N}(\Theta,\Theta^c)}{\partial \bw_k} =
-\sum_{n=1}^N \Big{[} \gamma \big{\{}\pi_k(\xx_n,\Theta^c)-g_k(\xx_n,\Theta)\big{\}}\\
&-\big{\{}y_n \pi_k(\xx_n,\Theta^c) (1-\sigma(y_n\bx_n^T\bw_k))\big{\}}\Big{]}\bx_n \label{Q-Grad}
\end{align}
 \item \textbf{Newton Method: }The Newton method updates the parameters as follows.
\begin{align}
\label{eq:Newton}
 \Theta^{c+1} = 
\Theta^c+\alpha^c (H^c)^{-1}
\gg^c
\end{align}
where $H^c$ is $dK\times dK $ Hessian matrix at $c^{th}$ iteration. $H$ is defined as follows.
\begin{equation*}
H  = \begin{pmatrix}
\frac{\partial^2 \mathcal{Q}_{N}(\Theta,\Theta^c)}{\partial \bw_1^2} & \frac{\partial^2 \mathcal{Q}_{N}(\Theta,\Theta^c)}{\partial \bw_1\bw_2} &\hdots &
\frac{\partial^2 \mathcal{Q}_{N}(\Theta,\Theta^c)}{\partial \bw_1\bw_K}\\
\frac{\partial^2 \mathcal{Q}_{N}(\Theta,\Theta^c)}{\partial \bw_2\bw_1}  & \frac{\partial^2 \mathcal{Q}_{N}(\Theta,\Theta^c)}{\partial \bw_2^2} &\hdots &
\frac{\partial^2 \mathcal{Q}_{N}(\Theta,\Theta^c)}{\partial \bw_2\bw_K}\\
\vdots & \vdots & \ddots & \vdots \\
\frac{\partial^2 \mathcal{Q}_{N}(\Theta,\Theta^c)}{\partial \bw_K\bw_1} & \frac{\partial^2 \mathcal{Q}_{N}(\Theta,\Theta^c)}{\partial \bw_K\bw_2} & \hdots &
\frac{\partial^2 \mathcal{Q}_{N}(\Theta,\Theta^c)}{\partial \bw_K^2}\end{pmatrix}
\end{equation*}
$\gg^c$ is as given earlier. Also,
\begin{align*}
 H_{kk}&=\frac{\partial^2 \mathcal{Q}_{N}(\Theta,\Theta^c)}{\partial \bw_k^2}\\
&=-\sum_{n=1}^N \gamma^2g_k(\xx_n,\Theta)(1-g_k(\xx_n,\Theta)\bx_n\bx_n^T\\
& -\sum_{n=1}^N\pi_k(\xx_n,\Theta^c) \sigma(\bw_k^T\bx_n)(1-\sigma(\bw_k^T\bx_n))\bx_n\bx_n^T\\
H_{kr}&= \frac{\partial^2 \mathcal{Q}_{N}(\Theta,\Theta^c)}{\partial \bw_k\bw_r}\\
&=\gamma^2\sum_{n=1}^N g_k(\xx_n,\Theta)g_r(\xx_n,\Theta)\bx_n\bx_n^T,\; \forall r\neq k
\end{align*}
\item \textbf{The BFGS Algorithm:}
In the BFGS method \cite{chong2013}, the Hessian matrix is not evaluated directly and
its inverse is approximated using rank-two updates specified by gradient evaluations.
Let $H^{-1}=B$, then 
\begin{flalign*}
 \nonumber B^{c+1}&=B^{c} +\left(1+\frac{\Delta {\gg^{c}}^T B^{c}\Delta \gg^{c}}{\Delta {\gg^{c}}^T \Delta \Theta^{c}}\right)
\frac{\Delta \Theta^{c} {\Delta \Theta^{c}}^T}{{\Delta \Theta^{c}}^T\Delta \gg^{c}}\\
&\;\;\;-\frac{B^{c}\Delta \gg^{c}
{\Delta \Theta^{c}}^T + \Theta^{c}{\Delta \gg^{c}}^T{B^{c}}^T}
{{\Delta \gg^{c}}^T\Delta \Theta^{c}}
\end{flalign*} 
where $\Delta {\gg^{c}}= \gg^{c+1}-\gg^{c}$ and
$\Delta {\Theta^{c}}= \Theta^{c+1}-\Theta^{c}$.
Then the update equation for weight vectors is as follows:
\begin{equation}
\label{eq:bfgs}
 \Theta^{c+1}=\Theta^{c} + \alpha^c B^{c}\gg^{c}
\end{equation}
where $\alpha^c$ is the step size.
\end{itemize}
In the above mentioned iterative optimization approaches, we need to choose the
step size appropriately to ensure the convergence of our approach.
There are many ways to find the step size. For our work, we use
the backtracking line search to find the step size.

\begin{algorithm}
\caption{{\bf P}olyhedral {\bf L}earning {\bf U}sing {\bf M}ixture of {\bf E}xperts (PLUME)} \label{algo:algo1}
\KwIn{$S=\{(\xx_1,y_1),\ldots,(\xx_N,y_N)\}$, $K$, $\epsilon$}
\KwOut{$\{\bw_1,\ldots,\bw_k\}$}
\Begin
{
\begin{itemize}
\item {\bf Initialize} $\bw_k^0,\;\forall k=1\ldots K$. Set $c=0$.
\item {\bf E-Step} For $k=1\ldots K$ and $n=1\ldots N$, find $\pi_k(\xx_n,\Theta^c)$ using eq.(\ref{eq:postprob})
\item {\bf M-Step} For $k=1\ldots K$, find $\bw_k^{c+1}$ as
\begin{equation*}
\Theta^{c+1}=\Theta^c + \alpha^c \dd^c
\end{equation*}
we use one of eq.(\ref{eq:gradient_ascent}), (\ref{eq:Newton}), (\ref{eq:bfgs}).
\item {\bf Termination}\\
	\lIf{$\mathcal{L}(\Theta^{c+1})-\mathcal{L}(\Theta^c)< \epsilon$}{
	    stop
	}
	\lElse{
	    $c=c+1$\;
	    go to E-Step
	}
\end{itemize}
}
\end{algorithm}

\section{Data Dependent Generalization Error Bounds}
\label{sec:GEB}
In this section, we will prove data dependent generalization error bounds. The generalization bounds show the utility
 of our formulation modelling the posterior probability function as a smooth function. These bounds shows that the method has proper asymptotic properties. 
 
 In this problem, we are finding the parameters by maximizing the likelihood which is equivalent to minimizing empirical risk under cross entropy loss $\phi (.,.)$ specified as below.
\begin{equation*}
    \phi(y, f(\xx)) = - y \ln p_{\Theta}( y | \xx ) - (1-y) \ln ( 1 - p_{\Theta}(y | \xx)  ) 
\end{equation*}
where $p_{\Theta}(y | \xx)$ is specified in eq.(\ref{eq:posterior}).
Define $\phi$-risk as 
\begin{equation}
\mathcal{R}_{\phi}(f) = \mathbb{E}\big[ \phi(y, f(\xx)) \big]
\end{equation}
 Define empirical $\phi-$risk as
 \begin{equation*}
 \hat{\mathcal{R}}_{\phi} (f) = \frac{1}{N} \sum_{n=1}^N \phi (y_{n} , f( \xx_{n}) )    
 \end{equation*}
 
Let $\mathcal{F}$ be a class of real-valued functions with domain $\mathbb{R}^d$. Here, we derive data dependent error bounds using Rademacher complexity. The empirical Rademacher complexity is defined as 
\begin{equation*}
\Rh_{N}(\mathcal{F}) = \mathbb{E}_{\mathbf{\epsilon}} \left\{ \sup_{f \in \mathcal{F}} \frac{1}{N} \sum_{n=1}^N \epsilon_{n} f(\xx_{n}) \right\}
\end{equation*}
where $\mathbf{\epsilon} = (\epsilon_{1}, \epsilon_{2}, ..., \epsilon_{N})$ is a vector of independent Rademacher random variable. The Rademacher complexity is the expected value of the empirical Rademacher complexity over all training sets of size $N$ i.e., $R_{N}(\mathcal{F}) = \mathbb{E}_{S} \big[  \Rh_{N}(\mathcal{F}) \big]$. \\[7pt]
We formalize our assumption as below to prove data dependent error bound. 
\begin{assumption} \label{assumption} Let $\| \ww_{k} \| \leq \mathbf{W}_{k}^{\max} \;\; \forall k \in \{1,...,K\}$ and $\| \xx \| \leq \mathbf{R},\;\forall \xx$. 
\end{assumption}
To simplify the notation, we introduce $\mathbf{W}^{\max}, \mathbf{W}^{\min}$ as $\mathbf{W}^{\max} = \max_{k} \; \mathbf{W}_{k}^{\max}$ and $\mathbf{W}^{\min} = \min_{k} \; \mathbf{W}_{k}^{\max}$.

We start our discussion regarding data dependent error bounds with the following result from 
\cite{Bartlett:2003:RGC:944919.944944}. 
\begin{result} 
\label{lemma:error-bound}
For every $\delta \in (0,1)$ and positive integer $N$, every $f \in \mathcal{F}$ satisfies 
\begin{equation}
\label{eq:risk-bound}
\mathcal{R}_{\phi}(f) \leq \hat{ \mathcal{R} }_{\phi} (f) + 2L_{\phi} \Rh_{N}(\mathcal{F}) + 3\phi_{\max} \sqrt{ \frac{ \text{ln} \frac{2}{\delta} }{2N} }
\end{equation}
with probability at least $1 - \delta$. 
\end{result}

In Result 2, we will first upper bound $\Rh_{N} ( \mathcal{F} )$ where $\mathcal{F} = \{ f:f(\xx) = \sum_{k=1}^K g_{k}(\xx, \Theta) \sigma(y \bw^T_{k} \bx ) \}$ . 
\begin{lemma}
\label{lemma:rh-rhk-bound}
Let $\mathcal{F}_{k} = \{ \; g_{k}(\xx, \Theta) \; \sigma(y \bw_{k}^T \bx) \; | \; \| \ww_{k} \| < \mathbf{W}_{k}^{\max}  \}$. Then, 
$\Rh_{N}(\mathcal{F}) \leq \sum_{k=1}^K \Rh_{N}(\mathcal{F}_{k})$. 
\end{lemma}
\begin{proof}
To simplify the notation, we write $\sup_{\ww_{1},...,\ww_{k}}$ as $\sup_{\ww}$. 
\begin{equation*}
\begin{aligned}
\Rh_{N}(\mathcal{F}) &= \mathbb{E}_{\epsilon} \left\{ \sup_{\Theta} \frac{1}{N} \sum_{n=1}^N \epsilon_{n} \sum_{k=1}^{K} g_{k}(\xx_{n}, \Theta) \sigma(y_n \bw_{k}^T \bx_n ) \right\} \\
&\leq \mathbb{E}_{\epsilon} \left\{ \sum_{k=1}^{K} \sup_{\Theta} \frac{1}{N}  \sum_{n=1}^N \epsilon_{n}  g_{k}(\xx_{n}, \Theta) \sigma(y_n \bw_{k}^T \bx_n) \right\} \\
&= \sum_{k=1}^{K} \mathbb{E}_{\epsilon} \left\{ \sup_{\Theta} \frac{1}{N}  \sum_{n=1}^N \epsilon_{n}  g_{k}(\xx_{n}, \Theta) \sigma(y_n \bw_{k}^T \bx_n) \right\} \\
&= \sum_{k=1}^{K} \Rh_{N} (\mathcal{F}_{k})
\end{aligned}
\end{equation*}
\end{proof}

Now, in order to bound $\Rh_{N}(\mathcal{F})$, using Lemma~\ref{lemma:rh-rhk-bound}, we will bound each of  $\Rh_N(\mathcal{F}_{k})$ individually.  
\begin{lemma}
\label{lemma:fk-bound}
Suppose Assumption \ref{assumption} holds then 
\begin{align*}
    \Rh_{N}(\mathcal{F}_{k}) \leq \frac{ 3\sqrt{(K-1)} }{K \sqrt{2}} \sum_{k=1}^K \frac{\mathbf{W}_{k}^{\max} {\mathbf{R}}}{\sqrt{N}} +  \frac{\mathbf{W}_{k}^{\max} {\mathbf{R}}}{\sqrt{N}}
\end{align*}
\end{lemma}
\begin{proof}
we decompose $\hat{R}_{N}(\mathcal{F}_{k})$ as follows.
\begin{equation*}
    \begin{aligned}
    \Rh_{N}(\mathcal{F}_{k}) &= \mathbb{E}_{\epsilon} \left\{ \sup_{\Theta} \frac{1}{N} \sum_{n=1}^N \epsilon_{n} g_{k}(\xx_{n}, \Theta) \sigma(y_{n}\bw_{k}^T\bx_{n}) \right\} \\
    =  & \; \mathbb{E}_{\epsilon} \left\{ \sup_{\Theta} \frac{1}{N} \sum_{n=1}^N \epsilon_{n} g_{k}(\xx_{n}, \Theta) (\sigma(y_{n}\bw_{k}^T\bx_{n}) - 0.5) \right. \\
       & \left. \;\;\;\;\;\; + \frac{1}{N} \sum_{n=1}^N 0.5 \epsilon_{n} g_{k}(\xx_{n}, \Theta) \right\} \\
    \leq  & \; \mathbb{E}_{\epsilon} \left\{ \sup_{\Theta} \frac{1}{N} \sum_{n=1}^N \epsilon_{n} g_{k}(\xx_{n}, \Theta) (\sigma(y_{n}\bw_{k}^T\bx_{n}) - 0.5) \right\} \\
    & \;\;\;\;\;\;   +  \mathbb{E}_{\epsilon} \left\{ \sup_{\Theta} \frac{1}{N} \sum_{n=1}^N 0.5 \epsilon_{n} g_{k}(\xx_{n}, \Theta) \right\}
    \end{aligned}
\end{equation*}
We define $\mathcal{G}_{k}^1  := \{ g_{k}(\xx, \Theta) \;\; | \;\; \| \ww_{k} \| \leq \mathbf{W}_{k}^{\max} \}$ and $\mathcal{G}_{k}^2 := \{ \sigma(y\bw_{k}^T\bx) \;\; | \;\; \| \ww_{k} \| \leq \mathbf{W}_{k}^{\max} \}$. We further define $\tilde{\mathcal{G}}_{k}^{2} = \{ (\sigma(y\bw_{k}^T\bx) - 0.5) \;\; | \;\; \| \ww_{k} \| \leq \mathbf{W}_{k}^{\max} \}$.  

We can easily check that $\tilde{\mathcal{G}}_{k}^{2}$ is closed under negation. Define the class $\mathcal{G}_{k}^{3} := \{ g : g(x_{1}, x_{2}) = g_{1}(x_{1}) g_{2}(x_{2}), g_{1} \in \mathcal{G}_{k}^{1}, g_{2} \in \tilde{\mathcal{G}}_{k}^{2} \}$.
Thus,
\begin{equation*}
        \Rh_{N}(\mathcal{F}_{k}) \leq \Rh_{N}(\mathcal{G}_{k}^{3}) +  0.5\Rh_{N}( \mathcal{G}_{k}^{1} ) 
\end{equation*}
Using Lemma~2 from \cite{mixture-of-experts} we observe that
\begin{align*}
\Rh_{N}(\mathcal{G}_{k}^{3}) \leq    \mathcal{M}_{1}\Rh_{N}(\mathcal{G}_{k}^{1}) + \mathcal{M}_{2}\Rh_{N}( \tilde{\mathcal{G}}_{k}^{2} )
\end{align*}
where $\mathcal{M}_{1} = \sup_{g_{1} \in \mathcal{G}_{k}^{1}} | g_{1}(\xx) |=1$ and $\mathcal{M}_{2} = \sup_{g_{2} \in \mathcal{G}_{k}^{2}} | g_{2}(\xx) |=0.5$. Thus,
\begin{equation*}
    \begin{aligned}
        \Rh_{N}(\mathcal{F}_{k}) &\leq  \Rh_{N}(\mathcal{G}_{k}^{1}) + 0.5 \Rh_{N}( \tilde{\mathcal{G}}_{k}^{2} ) + 0.5\Rh_{N}( \mathcal{G}_{k}^{1} ) \\
        & = 1.5 \Rh_{N}(\mathcal{G}_{k}^{1}) + 0.5 \Rh_{N}( \tilde{\mathcal{G}}_{k}^{2} ) \\
    \end{aligned}
\end{equation*}
As $\Rh_{N}(\tilde{\mathcal{G}}_{k}^{2})$ is the same as $\Rh_{N}(\mathcal{G}_{k}^{2})$, we can rewrite above equation as follows.
\begin{equation}
    \label{eq:Rhf-Rhs-Rhg}
    \Rh_{N}(\mathcal{F}_{k}) \leq 1.5 \Rh_{N}(\mathcal{G}_{k}^{1}) + 0.5 \Rh_{N}( \mathcal{G}_{k}^{2} )
\end{equation}
 To bound $\Rh_{N}(\mathcal{F}_{k})$, we will first bound $\Rh_{N}(\mathcal{G}_{k}^{1})$ using vector contraction inequality from \cite{maurer2016vector}.   One can verify that Lipschitz constant of $g_{k}(\xx, \Theta)$ with respect to $[\bw_{k}^{T} \bx]_{k=1}^{k=K}$ vector is $\frac{ \sqrt{K-1}}{K}$. Using vector contraction inequality from \cite{maurer2016vector}, we can write $\Rh_{N}(\mathcal{G}_{k}^{1})$ as
\begin{equation*}
    \begin{aligned}
        &\Rh_{N}(\mathcal{G}_{k}^{1}) \leq \sqrt{2} \Big( \frac{\sqrt{K-1}}{K} \Big) \mathbb{E}_{\epsilon} \left\{ \sup_{\Theta} \frac{1}{N} \sum_{n=1}^N \sum_{k=1}^K \epsilon_{nk} (\ww_{k}^T \xx_n +b_k) \right\} \\
        & \leq  \frac{\sqrt{2(K-1)}}{K}  \mathbb{E}_{\epsilon} \left\{ \sum_{k=1}^K \Big\{ \sup_{\ww_k,b_k} \; \frac{1}{N}  \sum_{n=1}^N  \epsilon_{nk}  (\ww_{k}^T\xx_n +b_k) \Big\}  \right\} \\
        & =  \frac{\sqrt{2(K-1)}}{K}  \mathbb{E}_{\epsilon} \left\{ \sum_{k=1}^K \Big\{ \sup_{\ww_k} \; \frac{1}{N} \ww_{k}^T \sum_{n=1}^N  \epsilon_{nk}  \xx_n \Big\}  \right\} \\
    \end{aligned}
\end{equation*}
where $\epsilon_{nk}$ is an independent Rademacher sequence. Using the Cauchy-Schwartz and Jensen inequalities,

\begin{align}
    \label{eq:Rgk-bound}
    \nonumber \Rh_{N}( \mathcal{G}_{k}^{1} ) &\leq \frac{\sqrt{2(K-1)}}{K}  \sum_{k=1}^K \mathbb{E}_{\epsilon} \left\{ \frac{1}{N} \mathbf{W}_{k}^{\max} \left\| \sum_{n=1}^N \epsilon_{nk} \xx_{n}  \right\| \right\} \\
    & \leq \frac{ \sqrt{2(K-1)} }{K} \sum_{k=1}^K \frac{\mathbf{W}_{k}^{\max} {\mathbf{R}}}{\sqrt{N}}
\end{align}

Now, we will bound second term $\Rh_{N}(\mathcal{G}_{k}^{2})$ of eq.(\ref{eq:Rhf-Rhs-Rhg}). 
\begin{equation*}
    \begin{aligned}
        \Rh_{N}( \mathcal{G}_{k}^{2} ) &= \mathbb{E}_{\epsilon} \left\{ \sup_{\ww_k} \frac{1}{N} \sum_{n=1}^N \epsilon_{n} \sigma (y_{n} \bw_{k}^T \bx_{n}) \right\}
    \end{aligned}
\end{equation*}
The Lipschitz constant of sigmoid function is 1. Using Theorem 12 of \cite{Bartlett:2003:RGC:944919.944944}, we can write above equation as
\begin{equation*}
    \begin{aligned}
        \Rh_{N}(\mathcal{G}_{k}^{2}) &\leq  \mathbb{E}_{\epsilon} \left\{ \sup_{\ww_k,b_k} \frac{1}{N} \sum_{n=1}^N \epsilon_{n} y_{n} (\ww_{k}^T \xx_{n}+b_k)  \right\} \\
         &=  \mathbb{E}_{\epsilon} \left\{ \sup_{\ww_k} \frac{1}{N} \sum_{n=1}^N \epsilon_{n} y_{n} \ww_{k}^T \xx_{n}  \right\} \\
    \end{aligned}
\end{equation*}
As $y_{n} \in \{+1, -1\}$, we can redefine $\epsilon_{n} = \epsilon_{n}y_{n}, \; \forall n$. Using new definition of $\epsilon_{n}$, we can rewrite above equation as
\begin{equation*}
    \begin{aligned}
        \Rh_{N}(\mathcal{G}_{k}^{2}) &\leq  \mathbb{E}_{\epsilon} \left\{ \sup_{\ww_k} \frac{1}{N} \sum_{n=1}^N \epsilon_{n} \ww_{k}^T \xx_{n}  \right\} \\
        &= \mathbb{E}_{\epsilon} \left\{ \sup_{\ww_k} \frac{1}{N} \ww_{k}^T \sum_{n=1}^N \epsilon_{n}  \xx_{n}  \right\} \\
    \end{aligned}
\end{equation*}
Using the Cauchy-Schwartz and Jensen inequalities,

    \begin{align}
    \label{eq:Rsk-bound}
        \Rh_{N}(\mathcal{G}_{k}^{2}) \leq \frac{1}{N} \mathbb{E}_{\epsilon} \left\{ \mathbf{W}_{k}^{\max} \left\| \sum_{n=1}^N \epsilon_{n} \xx_{n} \right\| \right\} 
        \leq \frac{\mathbf{W}_{k}^{\max} {\mathbf{R}}}{\sqrt{N}} 
    \end{align}

Putting values of $\Rh_{N}(\mathcal{G}_{k}^{1})$ and $\Rh_{N}(\mathcal{G}_{k}^{2})$ from eq.(\ref{eq:Rgk-bound}) and (\ref{eq:Rsk-bound}) in eq.(\ref{eq:Rhf-Rhs-Rhg}), we will get desired result for Lemma \ref{lemma:fk-bound}.
\begin{equation*}
    \begin{aligned}
        \Rh_{N}(\mathcal{F}_{k}) \leq \frac{ 3\sqrt{(K-1)} }{K \sqrt{2}} \sum_{k=1}^K \frac{\mathbf{W}_{k}^{\max} {\mathbf{R}}}{\sqrt{N}} +  \frac{\mathbf{W}_{k}^{\max} {\mathbf{R}}}{\sqrt{N}}
    \end{aligned}
\end{equation*}
\end{proof} 

Now, we will present main theorem containing data dependent generalization error bounds for our approach. 
\begin{theorem}
\label{thm-data-dependent-bound}
Suppose assumption \ref{assumption} holds. Then, for any function $f$, there holds 
\begin{equation*}
    \begin{aligned}
        \mathcal{R}(f) \leq \hat{\mathcal{R}}_{\phi}(f, S) + c_{1} \sum_{k=1}^K \frac{ \mathbf{W}_{k}^{\max} \mathbf{R} }{ \sqrt{N} } + c_{2} \sqrt{ \frac{ \text{ln} \frac{2}{\delta}}{ 2N } }
    \end{aligned}
\end{equation*}
with probability at least $1-\delta$ where $c_{1} = (1+e^{ \mathbf{W}^{\max} \mathbf{R}}) e^{ \gamma (\mathbf{W}^{\max} + \mathbf{W}^{\min})  \mathbf{R} } \left( \frac{3 \sqrt{K-1}}{\sqrt{2}} + 1 \right)$ and $c_{2} = 3(1+e^{ \mathbf{W}^{\max} \mathbf{R}}) e^{ \gamma (\mathbf{W}^{\max} + \mathbf{W}^{\min})  \mathbf{R} }$.
\end{theorem}
\begin{proof}
Using results of Lemma \ref{lemma:rh-rhk-bound} and Lemma \ref{lemma:fk-bound}, we can bound $\Rh_{N}( \mathcal{F} )$ as follows.
\begin{equation*}
\begin{aligned}
    \Rh_{N}( \mathcal{F} ) &\leq \sum_{k=1}^K \Rh_{N}( \mathcal{F}_{k} ) \\
    &\leq \frac{ 3\sqrt{(K-1)} }{\sqrt{2}} \sum_{k=1}^K \frac{\mathbf{W}_{k}^{\max} {\mathbf{R}}}{\sqrt{N}} + \sum_{k=1}^K \frac{\mathbf{W}_{k}^{\max} {\mathbf{R}}}{\sqrt{N}} \\
    &\leq \left( \frac{ 3\sqrt{(K-1)} }{\sqrt{2}} + 1 \right) \sum_{k=1}^K \frac{\mathbf{W}_{k}^{\max} {\mathbf{R}}}{\sqrt{N}}
\end{aligned}
\end{equation*}
 If assumption \ref{assumption} holds then the Lipschitz constant of cross-entropy loss is bounded by $ (1+e^{ \mathbf{W}^{\max} \mathbf{R}}) e^{ \gamma (\mathbf{W}^{\max} + \mathbf{W}^{\min})  \mathbf{R} }$. When assumption \ref{assumption} holds, the maximum value of cross-entropy loss will be bounded by $(1+e^{ \mathbf{W}^{\max} \mathbf{R}}) e^{ \gamma (\mathbf{W}^{\max} + \mathbf{W}^{\min})  \mathbf{R} }$. Putting $\hat{R}_{N}(\mathcal{F})$ bound in eq.(\ref{eq:risk-bound}), we will get the following generalization error bound for our approach. 
\begin{equation*}
    \begin{aligned}
        \mathcal{R}(f) \leq \hat{\mathcal{R}}_{\phi}(f, S) + c_{1} \sum_{k=1}^K \frac{ \mathbf{W}_{k}^{\max} \mathbf{R} }{ \sqrt{N} } + c_{2} \sqrt{ \frac{ \text{ln} \frac{2}{\delta}}{ 2N } }
    \end{aligned}
\end{equation*}
where $c_{1} = (1+e^{ \mathbf{W}^{\max} \mathbf{R}}) e^{ \gamma (\mathbf{W}^{\max} + \mathbf{W}^{\min})  \mathbf{R} } \left( \frac{3 \sqrt{K-1}}{\sqrt{2}} + 1 \right)$ and $c_{2} = 3(1+e^{ \mathbf{W}^{\max} \mathbf{R}}) e^{ \gamma (\mathbf{W}^{\max} + \mathbf{W}^{\min})  \mathbf{R} }$.
\end{proof}

Thus, the bound decreases as $O(N^{-1/2})$.

\section{Experiments}\label{sec:Exp}
To show the effectiveness of the proposed approach PLUME (given in Algorithm~\ref{algo:algo1}), we compare it with several state of the art approaches.
We compare our approach with CART (classification and regression trees), which is a top-down oblique decision tree algorithm. We also compare with two fixed structure approaches: (a) SPLA1 (single polyhedral learning algorithm), the algorithm proposed in \citet{Naresh10} and (b) CPM (convex polytope machine) for learning polyhedral classifiers discussed in \cite{NIPS2014_5511}. We also present comparison with SVM which is a generic pattern recognition methods. 

We test the performance of our approach on several real world datasets downloaded from UCI ML repository \cite{Asuncion+Newman:2007}. We show results on "Ionosphere", "Heart", "ILPD", "Pima Indian" and "Adult" datasets.

We implemented our proposed approach PLUME (Algorithm~\ref{algo:algo1}) using gradient ascent, Newton method
and BFGS in Python. We also implemented SPLA1 \citep{Naresh10} in Python. For CPM (convex polytope machine) we used the package made available by the authors \citep{Alkant2014}. For SVM and CART, we used Scikit-learn library \citep{scikit-learn} in Python. 

\begin{table*}[h!]
\begin{center}
\begin{tabular}{| c | c | c | c | c|}
\hline 
Dataset & Method & Accuracy & Time(Sec.) & \#hyp. \\
\hline
Ionosphere & SPLA1 (GA) & 90.78 $\pm$ 1.48 & 0.19 $\pm$ 0.008 & 3\\
& SPLA1 (Newton) & 88.36 $\pm$ 1.24 & 0.20 $\pm$ 0.03 & 2\\
& SPLA1 (BFGS) & 88.97 $\pm$ 1.51 & 0.09 $\pm$ 0.004 & 2\\
& PLUME (GA) & 89.86 $\pm$ 1.14 & 1.89 $\pm$ 1.06 & 3\\
& PLUME (Newton) & 88.92 $\pm$ 1.33 & 5.91 $\pm$ 2.32 & 3\\
& PLUME (BFGS) & 88.10 $\pm$ 1.65 & 0.027 $\pm$ 0.002 & 2\\
& CPM & 85.61 $\pm$ 1.79 & 29.35 $\pm$ 16.54 & 2\\
& CART & 88.39 $\pm$ 1.18 & 0.006 $\pm$ 0.0004 & \\
& SVM & 94.31 $\pm$ 0.88 & 0.006  & \\
\hline
Heart & SPLA1 (GA) & 83.78 $\pm$ 1.50 & 0.14 $\pm$ 0.04 & 2\\
& SPLA1 (Newton) & 83.70 $\pm$ 1.72 & 0.04 $\pm$ 0.0004 & 2\\
& SPLA1 (BFGS) & 83.63 $\pm$ 1.30 & 0.09 $\pm$ 0.0008 & 2\\
& PLUME (GA) & 83.25 $\pm$ 2.19 & 0.72 $\pm$ 0.31 & 2\\
& PLUME (Newton) & {\bf 84.07} $\pm$ 2.32 & 1.09 $\pm$ 0.23 & 2\\
& PLUME (BFGS) & 83.05 $\pm$ 2.13 & 0.02 $\pm$ 0.001 & 2 \\
& CPM & 77.44 $\pm$ 4.29 & 23.28 $\pm$ 0.66 & 4 \\
& CART & 73.53 $\pm$ 2.89 & 0.001 $\pm$ 0.0001 &\\
& SVM & 78.93 $\pm$ 1.88 & 0.009 $\pm$ 0.0002 &\\
\hline
ILPD & SPLA1 (GA) & 72.32 $\pm$ 1.38 & 0.97 $\pm$ 0.11 & 3\\
& SPLA1 (Newton) & 71.19 $\pm$ 1.61 & 0.23 $\pm$ 0.11 & 2\\
& SPLA1 (BFGS) & 71.12 $\pm$ 1.54 & 0.18 $\pm$ 0.002 & 2\\
& PLUME (GA) & {\bf 72.45} $\pm$ 1.51 & 0.57 $\pm$ 0.38 & 2 \\
& PLUME (Newton) & 70.93 $\pm$ 1.32 & 5.87 $\pm$ 0.10 & 2\\
& PLUME (BFGS) & 70.90 $\pm$ 2.21 & 0.0403 $\pm$ 0.001 & 2\\
& CPM & 67.71 $\pm$ 2.38 & 21.17 $\pm$ 18.27 & 2\\
& CART & 65.07 $\pm$ 1.73 & 0.003 $\pm$ 0.0001 &\\
& SVM & 69.59 $\pm$ 1.34 & 0.03 $\pm$ 0.0006 & \\
\hline
Pima & SPLA1 (GA) & 76.77 $\pm$ 1.44 & 0.16 $\pm$ 0.12 & 3\\
Indian & SPLA1 (Newton) & 77.40 $\pm$ 1.04 & 0.17 $\pm$ 0.04 & 2\\
& SPLA1 (BFGS) & 76.77 $\pm$ 0.97 & 0.28 $\pm$ 0.01 & 2\\
& PLUME (GA) & {\bf 77.95} $\pm$ 1.29 &  3.09 $\pm$ 0.70 & 2\\
& PLUME (Newton) & 77.30 $\pm$ 1.17 & 0.39 $\pm$ 0.005 & 2 \\
& PLUME (BFGS) & 77.04 $\pm$ 1.91 & 0.053 $\pm$ 0.002 & 2 \\
& CPM & 69.44 $\pm$ 1.11 & 26.24 $\pm$ 0.55 & 8\\
& CART  & 69.74 $\pm$ 1.66 & 0.004 $\pm$ 0.0001 &\\
& SVM & 75.80 $\pm$ 1.45 & 0.03 $\pm$ 0.0007 &\\
\hline
Adult & SPLA1 (GA) & 77.13 $\pm$ 0.64 & 4.75 $\pm$ 1.17 & 2\\
& SPLA1 (Newton) & 82.89 $\pm$ 0.74 & 9.96 $\pm$ 1.13 & 2\\
& SPLA1 (BFGS) & 75.59 $\pm$ 0.61 & 0.52 $\pm$ 0.005 & 2\\
& PLUME (GA) & 76.92 $\pm$ 0.97 & 11.62 $\pm$ 1.82 & 2\\
& PLUME (Newton) & 80.53 $\pm$ 0.73 & 1.79 $\pm$ 0.02 & 2\\
& PLUME (BFGS) &  81.01 $\pm$ 0.75 & 0.402 $\pm$ 0.004 & 2\\
& CPM & 72.4 $\pm$ 6.47 & 29.03 $\pm$ 0.27 & 16\\ 
& CART & 80.84 $\pm$ 0.75 & 0.012 $\pm$ 0.0002 & \\
& SVM & 75.67 $\pm$ 0.66 & 0.75 $\pm$ 0.009 &\\
\hline 
\end{tabular}
\caption{Comparison results of PLUME with other approach}
\label{Table:results}
\end{center}
\end{table*}

\subsection*{Simulation Results}
We use following acronyms for different approaches. GA stands for gradient ascent, Newton stands for Newton method, BFGS is quasi-Newton approach. 

All the results were generated by repeating 10-fold cross validation tests 10 times. We report the average accuracy and standard deviation over 10 runs. Also, to find the best parameters for every algorithm, we used 10-fold cross validation. For example, number of hyperplanes in all the polyhedral learning approach, kernel parameters in SVM were found using cross validation. For SVM, we used Gaussian kernel for all the datasets. 

Comparison results are given in Table~\ref{Table:results}. We see that except for the Ionosphere dataset, our approach performs with a significantly better accuracy than SVM. As compared to CART, PLUME performs better accuracy wise on Heart, ILPD and Pima datasets. On Ionosphere and Adult datasets, PLUME accuracy is comparable to CART. Thus, PLUME performs better or comparable to the generic methods whenever the best classifier is close to the polyhedral classifier. 

With respect to CPM, PLUME always performs better in terms of accuracy with significant margins. Moreover, CPM always require more number of hyperplanes to achieve these accuracies compared to PLUME. Thus, PLUME always learn less complex classifiers compared to CPM. 

Compared to SPLA1, PLUME performs eqally good on all the datasets. Except for the Adult dataset, one of the PLUME variant (GA, Newton or BFGS) achieves better accuracy than all the SPLA1 variants. Thus, PLUME performs at par compared to other state of the art polyhedral learning approaches. 

Now we will now discuss the comparison results on training time. We can clearly see that CPM takes significantly more time compared PLUME for all the datasets. As compared to SPLA1, PLUME takes slightly more training time. This happens because PLUME does soft partitioning of the training set. Which means, every example participates in classifier learning corresponding to each expert (with different probabilities defined by $\pi(\xx,\Theta)$).

As compared to SVM, PLUME takes more time to train, which is clearly understandable from the modeling. SVM solves convex while we do not and so we expect SVM to do well in terms of time. But PLUME learns the polyhedral set as an intersection of half spaces and this gives a much better geometric insight in the original feature space.    
 
\section{Conclusions}\label{conclusions}
We have proposed a novel approach PLUME for learning polyhedral classifiers using mixture of experts. We have proposed an EM based learning algorithm to find the parameters of the classifier. We have also derived the data dependent generalized error bound for the proposed model. We have also shown experimentally that the proposed approach performs comparable to the state of the art polyhedral learning approaches as well as generic approaches (SVM, CART etc.) on various datasets. 

\bibliography{poly1}
\bibliographystyle{icml2019}


\end{document}